\documentclass[twoside,11pt]{article}

\usepackage{blindtext}

\usepackage{jmlr2e}
\usepackage[utf8]{inputenc}

\usepackage{pgf,tikz}

\usepackage{hyperref}
\usepackage{graphicx} 
\usepackage{booktabs} 

\usepackage{amsfonts, mathtools} 
\usepackage{amsmath}
\usepackage{amssymb}
\usepackage{xcolor}
\usepackage{bbm}
\usepackage{color}

\usepackage{float}

\usepackage{comment}

\usepackage{pdfsync}

\newcommand{\R}{\mathbb{R}}

\renewcommand{\P}{\mathbb{P}}

\newcommand{\E}{\mathbb{E}}

\newtheorem{assumption}[theorem]{Assumption}

\ShortHeadings{Convergence of continuous-time stochastic gradient descent}{Gabor Lugosi and Eulalia Nualart}
\firstpageno{1}

\begin{document}

\title{Convergence of continuous-time stochastic gradient descent with applications to deep neural networks}

\author{\name Gabor Lugosi \email gabor.lugosi@upf.edu \\
       \addr Universitat Pompeu Fabra and Barcelona School of Economics\\ Department of Economics and Business \\
        Ram\'on Trias Fargas 25-27, 08005, Barcelona, Spain \\
       ICREA, Pg. Lluís Companys 23, 08010 Barcelona, Spain 
       \AND
       \name Eulalia Nualart\email eulalia.nualart@upf.edu \\
       \addr Universitat Pompeu Fabra and Barcelona School of Economics \\ Department of Economics and Business \\
        Ram\'on Trias Fargas 25-27, 08005, Barcelona, Spain }

\editor{}

\maketitle

\begin{abstract}
We study a continuous-time approximation of the stochastic gradient descent process for
minimizing the population expected loss in learning problems. The main results establish general sufficient
conditions for the convergence, extending the results of \cite{C22} established for (nonstochastic)
gradient descent. We show how the main result can be applied to the case of overparametrized 
neural network training.
\end{abstract}


\section{Introduction}
 
Stochastic gradient descent ({\sc sgd}) is a simple yet remarkably powerful optimization method
that has been widely used in machine learning, most notably in the training
of large neural networks. Indeed, {\sc sgd} has played a central role in the spectacular
success of deep learning. Despite its importance, the method remains far from fully understood,
and significant effort has been devoted to explaining why large neural networks trained
by stochastic gradient descent learn so efficiently and generalize so well. 

We now describe a general setup that encompasses a broad class
of problems in machine learning. 
Let $\ell:\R^D \times \R^d\rightarrow [0, \infty)$ be a \emph{loss
  function} that assigns a nonnegative value to any pair $(w,z)$, where
$w\in \R^D$ is a \emph{parameter} to be learned and $z\in \R^d$ is an
\emph{observation}. We assume throughout that $\ell$ is twice
continuously differentiable in its first argument.
Let $Z$ be a random vector taking values in $\R^d$. The goal is to
minimize the population expected loss (or \emph{population risk}) $f(w)=\E[\ell(w,Z)]$ over
$w\in \R^D$. To this end, one has access to training data in the form
of a sequence $Z_0,Z_1,Z_2,\ldots$ of independent, identically distributed copies of $Z$.

Stochastic gradient descent ({\sc sgd}) is the iterative
optimization algorithm defined by an arbitrary initial value $w_0 \in
\R^D$ and a \emph{step size} $\eta>0$, which updates for $k=0,1,2,\ldots$ as
\begin{equation} \label{w0}
w_{k+1}=w_k-\eta \nabla \ell(w_k, Z_k)~,
\end{equation} 
where $\nabla$ denotes the derivative with respect to $w$. Clearly,
$$
\E[\nabla \ell(w_k, Z_k) \mid w_k]=\nabla f(w_k)~.
$$
In this paper, we study a continuous-time approximation of the stochastic gradient descent
process. 
Several approximations have been proposed in the literature. We follow
the model introduced by 
\cite{P20}, 
which approximates the {\sc sgd} recursion (\ref{w0}) by the Langevin-type
continuous-time stochastic differential equation ({\sc sde})
\begin{equation} \label{w2}
dw_t=-\nabla f(w_t) \, dt+\sqrt{\eta} \, \sigma(w_t) \, dB_t~, 
\end{equation}
for $t\ge 0$,
 where $w_0 \in \R^D$, $B_t$ is a $D$-dimensional Brownian motion,  $\eta>0$ is a fixed parameter that acts as the variance of the noise term, and $\sigma: \R^D \rightarrow \R^D \times \R^D$ is a $D\times D$ matrix
defined as the unique square root of the covariance matrix $\Sigma(w)= \text{Cov}(\nabla \ell(w,Z))$ of the random vector $\nabla \ell(w,Z)$, that is,
\begin{equation*} \begin{split}
\sigma(w)(\sigma(w))^\top= \Sigma(w)~.
\end{split}
\end{equation*}
For the heuristics behind the approximation of the discrete-time process (\ref{w0}) by  (\ref{w2}), we refer the reader to \cite{P20}.
We investigate convergence properties of \eqref{w2}, as $t \rightarrow \infty$, for functions $f:\R^D\to [0,\infty)$ and $\sigma:\R^D\to S_+^D$ defined via a loss function as above, \textcolor{black}{where $S_+^D$ is defined in Subsection \ref{sub:not} below.}

General sufficient conditions for convergence of the ``noiseless'' process---corresponding to $\eta=0$ in (\ref{w2})---to a global minimum of $f$
were established by \cite{C22}. While the behavior of gradient descent is well understood 
when $f$ is convex (\cite{Nes13}), 
Chatterjee's conditions extend significantly beyond convexity.
The main goal of this paper is to extend Chatterjee's results
to the stochastic model (\ref{w2}). The presence of noise introduces new challenges, and addressing
these is our main contribution. \textcolor{black}{It is important to highlight that in this work we study minimization of the \emph{population risk} $f(w)=\E[\ell(w,Z)]$, rather than its empirical counterpart. It is the population risk that is relevant for the performance of the learning algorithm, as, in general, a small empirical risk does not imply good generalization.}

The rest of the paper is organized as follows. 
\textcolor{black}{In Section \ref{sec:prel} we introduce the main assumptions, notation, and elements of stochastic calculus that are relevant for our techniques.}
In Section \ref{sec:main} we present the main result of the paper. In particular, Theorem \ref{t2} shows that, under Chatterjee's conditions,
together with additional assumptions on the noise $\sigma(\cdot)$, if the process is initialized sufficiently close to a global minimum, then,
with high probability, the trajectory $w_t$ converges to the set of global minima of $f$.
In Section \ref{sec:lit} we review related literature. 
In Section \ref{sec:app} we illustrate how the main result can be applied to the training of overparameterized 
neural networks.
All proofs are collected in Section \ref{sec:proofs}.

\section{Preliminaries and assumptions} \label{sec:prel}
\subsection{Notation} \label{sub:not}
$\Vert \cdot \Vert$ denotes the Euclidean norm 
in both $\R^D$ and $\R^d$. All random variables are defined on a complete probability space $(\Omega, \mathcal{F}, \P)$, and we denote by
$\E[\cdot]$ the expectation with respect to $\P$. We let $(\mathcal{F}_t)_{t \geq 0}$
be the minimal augmented filtration generated by the $D$-dimensional Brownian motion $(B_t)_{t \geq 0}$, satisfying the usual conditions.
For any integer $k \geq 1$, we denote by $\mathcal{C}^2(\R^k)$ the set of
twice continuously differentiable functions $g:\R^k \rightarrow \R$.
If $g \in \mathcal{C}^2(\R^D)$,
we write $\nabla g(w)$ for its gradient and 
$H g(w)$ for its $D\times D$ Hessian matrix.
We denote by $B_{r}(w)\subset \mathbb{R}^D$ the closed Euclidean ball of radius $r>0$ centered at $w$.
For any square matrix $M$, we write $\text{Tr}(M)$ for its trace, and $\lambda_{\min}(M)$ and $\lambda_{\max}(M)$ for its smallest and largest eigenvalues, respectively. For any matrix $M$, we denote by $M^\top$ its transpose. 
If $M$ is a $D \times D$ matrix, then $M_1,\ldots,M_D$ denote its column vectors. We let $S_+^D$ denote the set 
of positive definite $D\times D$ matrices. We say that a function $g:\R^D \rightarrow \R^D$ is locally Lipschitz continuous if, for any compact set $K\subset \R^D$, there exists a constant $\textnormal{Lip}(g,K)>0$ such that for all $x,y \in \R^D$,
\begin{equation} \label{lip}
\Vert \nabla g (x)- \nabla g(y) \Vert \leq \textnormal{Lip}(g,K)  \Vert x-y \Vert.
\end{equation}
If $a,b \in \R$, we set $a \wedge b:=\min\{a,b\}$.

\subsection{Assumptions}
In this subsection we state the key assumptions needed to obtain convergence of the process (\ref{w2}) as $t \rightarrow \infty$ to
a global minimizer of the function $f$.

Our first assumption is a regularity condition on the function $f$, namely a ``locally Lipschitz'' condition for $\nabla f$, \textcolor{black}{whose definition is given in (\ref{lip})}. 
This mild assumption guarantees that equation (\ref{w2}) admits a unique local solution, as explained below.
It is important
to emphasize that we do not require $\nabla f$ to be globally Lipschitz continuous, since this would exclude some important applications 
in machine learning.
\begin{assumption} \label{a1}
The functions $\nabla f, \sigma_1,\ldots, \sigma_d: \R^D \rightarrow \R^D$
are locally Lipschitz continuous.
\end{assumption}

Under Assumption \ref{a1}, it is well known (see, e.g., \cite[Theorem 2.8, page 154]{Mao1}, \cite{Mao2}) that 
for any initialization $w_0\in \R^D$,
there exists a unique maximal local solution to equation (\ref{w2}) up to its (random) blow-up time  
\[
T:=T(w_0) = \sup\{ t>0 : \| w_t \| < \infty \}.
\]
This means that there exists a unique continuous $\mathcal{F}_t$-adapted Markov process $(w_t)_{t\geq 0}$ satisfying the integral equation
\[
w_t = w_0 - \int_0^t \nabla f(w_s) \, ds + \sqrt{\eta} \int_0^t \sigma(w_s) \, dB_s~,
\]
for all $t <T$ a.s., where the stochastic integral is understood in the It\^o sense.
Moreover, if $T<\infty$, then
\[
\limsup_{t \rightarrow T} \Vert w_t \Vert = \infty~.
\]

We now introduce our second assumption.  
Recall that our main goal is to derive sufficient conditions on the function $f$ under which the solution $w_t$ converges to a point where 
$f$ attains its minimum. An obvious necessary condition for convergence is that the norm of $\sigma(w)$ tends to zero
as $w$ approaches the set of minimizers. In other words, we assume that $f$ reaches its minimum value, which we normalize to be zero:

\textcolor{black}{
\begin{assumption} \label{a1B}
There exists $w \in \R^D$ such that $f(w)=0$.
\end{assumption}
Since $f(w)=\E [\ell(w,Z)]$ and $\ell$ is non-negative,
Assumption \ref{a1B} is equivalent to the interpolation assumption
$$
\text{there exists } w\in \R^D \text{ such that } 
\ell(w,Z)=0 \quad \text{almost surely.}$$ }

\textcolor{black}{
In many machine learning applications, it is natural and reasonable to assume that the learning problem is ``noiseless'', and the hypothesis class is sufficiently rich. Under such circumstances,  Assumption \ref{a1B} holds.
}

An immediate and simple consequence of Assumption \ref{a1B} is that if $f$ attains its minimum value at a finite time, 
then the solution of the process remains at that point forever, almost surely:

\begin{lemma} \label{l1}
Consider the {\sc sde} \textnormal{(\ref{w2})} initialized at some $w_0 \in \R^D$, and
suppose that Assumptions \ref{a1} and \ref{a1B} hold.
If for some $t \in [0,T)$ we have $f(w_t)=0$, then $T=\infty$ and for all $s > t$, $w_s = w_t$.
\end{lemma}

\subsection{Preliminaries on It\^o's stochastic calculus} \label{sec:pre}

In this subsection we introduce some important notation together with preliminary lemmas 
from stochastic calculus that play a key role in formulating and proving our convergence result. 
The main tool is the theory of It\^o's stochastic integration; see, for instance, the monograph by \cite{Mao1} for an introduction to this topic. 
We begin by recalling the multi-dimensional It\^o formula, which can be found in Theorem 6.4, page 36, of this monograph.

\begin{theorem}[Multi-dimensional It\^o formula] \label{itofor}
Let $x_t$ be a $D$-dimensional It\^o process defined up to an $\mathcal{F}_t$-stopping time $\rho$. 
That is, $x_0 \in \R^D$ and $x_t$ satisfies the stochastic differential equation
\begin{equation*}
dx_t=u_t dt+ v_t dB_t, \quad \text{a.s. for all }  0\leq t<\rho,
\end{equation*}
where $u_t$ is an $\R^D$-valued measurable $\mathcal{F}_t$-adapted process defined a.s. for all $0\leq t<\rho$ such that
\begin{equation} \label{h1}
\int_0^{\rho} \Vert u_t \Vert dt <\infty  \quad \text{a.s.},
\end{equation}
and
$v_t$ is an $\R^D \times \R^D$-valued measurable $\mathcal{F}_t$-adapted process defined a.s. for all $0\leq t<\rho$ such that
\begin{equation} \label{h2}
\E\left[\int_0^{\rho} \textnormal{Tr}(v_t^\top v_t)dt \right]<\infty.
\end{equation}
Let $V \in \mathcal{C}^2(\R^D)$.  
Then $V(x_t)$ is an It\^o process defined a.s. for all $0\leq t<\rho$ with stochastic differential
\begin{equation}\label{if} \begin{split}
dV(x_t)=\left((\nabla V(x_t))^\top u_t+\frac12 \textnormal{Tr}(v_t^\top HV(x_t) v_t)\right) dt
+(\nabla V (x_t))^\top v_t dB_t,
\end{split}
\end{equation}
a.s. for all $0\leq t<\rho$.
\end{theorem}
\begin{remark} 
Although \cite[Theorem 6.4, page 36]{Mao1} is stated only in the case $\rho=\infty$ a.s., 
the result holds a.s. for all $t \geq 0$. Hence, for any fixed $\omega \in \Omega$, the statement is valid for all $t \geq 0$, 
and in particular it also applies to an $\mathcal{F}_t$-stopped process. 
\end{remark}

By Assumption \ref{a1}, the solution to our {\sc sde} \eqref{w2} is an It\^o process up to its blow-up time $T$, and therefore exists only locally. 
Moreover, by Assumption \ref{a1B} and Lemma \ref{l1}, we restrict attention to the case where $f(w_t)$ does not reach its minimum value (zero) in finite time. 
With this in mind, we define the two $\mathcal{F}_t$-stopping times
\[
\tau_r:=\tau_r(w_0) = \inf\{ t>0 : w_t \notin B_r(w_0)\}, 
\qquad
\tau:=\tau(w_0) = \inf\{ t>0 : f(w_t)=0 \}.
\]
That is, $\tau_r$ is the first time the process leaves the ball of radius $r$ around its initial point, 
and $\tau$ is the first time $f(w_t)$ attains its minimum value.

A first key step in proving convergence of $w_t$ is to study the local stability of $f(w_t)$ 
by adapting the theory of Lyapunov exponents developed in \cite[Chapter 2]{Mao1} to our setting.
To this end, we apply the multi-dimensional It\^o formula to the stopped process 
$\log f(w_{t\wedge \tau_r \wedge \tau})$; see the proof of Lemma \ref{p1} below for details. 
Specifically, applying formula (\ref{if}) with $V=\log f$ and $x_t=w_{t \wedge \tau_r \wedge \tau}$ yields, in integral form,
\begin{equation} \label{md}
\log f(w_{t \wedge \tau_r \wedge \tau})
   =\log f(w_0)-\int_0^{t \wedge \tau_r \wedge \tau} (a (w_s)-\eta g(w_s))  ds +M_{t}-\frac{1}{2} \langle M \rangle_{t},
\end{equation}
where $(M_{t})_{t \geq 0}$ is the stopped $\mathcal{F}_t$-martingale
\begin{equation*} 
M_{t}:= \sqrt{\eta}\int_0^{t \wedge \tau_r \wedge \tau} \frac{(\nabla f(w_s))^\top \sigma(w_s)}{f(w_s)}   dB_s,
\end{equation*}
with quadratic variation (see \cite[Theorem 5.21, page 28]{Mao1}) given by
$$
\langle M \rangle_{t}= \eta\int_0^{t \wedge \tau_r \wedge \tau} 
\frac{\textnormal{Tr}\left((\sigma(w_s))^\top \nabla f(w_s) 
(\nabla f(w_s))^\top \sigma(w_s)\right) }{f^2(w_s)} ds.
$$
For $w \in \R^d$, we set
\begin{equation} \label{eq:a}
 a(w):=\frac{ \Vert \nabla f(w) \Vert^2}{f(w)}, 
 \qquad
 g(w):=\frac{ \textnormal{Tr}((\sigma(w))^\top Hf(w)\sigma(w))}{2f(w)}.
\end{equation}

To upper bound the right-hand side of (\ref{md}), we define
\begin{equation} \label{eq:A}
A_{\min}(r,w_0):=\inf_{w \in B_{r}(w_0), f(w) \neq 0} a(w),
\qquad
G_{\max}(r,w_0):=\sup_{w \in B_r(w_0), f(w) \neq 0} g(w).
\end{equation}
If $f(w)=0$ for all $w \in B_r(w_0)$, we set $A_{\min}(r,w_0)=\infty$.
For $\eta \geq 0$, we also define 
$$\theta(r,w_0,\eta):= A_{\min}(r,w_0)-\eta G_{\max}(r,w_0).$$

We then obtain the following two results, whose proofs are postponed to Section \ref{sec:proofs}. 
The first provides a local exponential upper bound for $f$, while the second shows that 
if $f$ does not reach its minimum in finite time, then $f$ decays exponentially to zero at infinity.
\begin{lemma} \label{p1}
Consider the {\sc sde} \textnormal{(\ref{w2})} initialized at some $w_0 \in \R^D$, and
suppose that Assumptions \ref{a1} and \ref{a1B} hold. Then, for all $r>0$ and $\eta \geq 0$, almost surely for all $t>0$,
\begin{equation*} \begin{split}
f(w_{t \wedge \tau_r \wedge \tau})\leq f(w_0)e^{-(t \wedge \tau_r \wedge \tau)\theta(r,w_0,\eta) } 
e^{M_{t}-\frac{1}{2}\langle M \rangle_{t}}.
\end{split}
\end{equation*}
\end{lemma}

\begin{lemma} \label{t1}
Consider the {\sc sde} \textnormal{(\ref{w2})} initialized at some $w_0 \in \R^D$, and
suppose that Assumptions \ref{a1} and \ref{a1B} hold. Then, for all $r>0$ and $\eta \geq 0$, almost surely on the event $\{\tau_r \wedge \tau =\infty\}$, 
$$
\limsup_{t \rightarrow \infty} \frac{\log f(w_{t})}{t}\leq -\theta(r,w_0,\eta).
$$
\end{lemma}

Recall from \cite[Chapter 2]{Mao1} that the quantity 
$\limsup_{t \rightarrow \infty} \frac{\log f(w_{t})}{t}$ is called the Lyapunov exponent of the process $f(w_{t})$.

A second key step in proving convergence of the {\sc sde} \textnormal{(\ref{w2})}
is to control locally the quadratic variation of the It\^o integral, given by
$$
\E\left[  \int_{0}^{t \wedge \tau_r \wedge \tau} \text{Tr}(\sigma(w_s)^\top\sigma(w_s)) ds  \right].
$$
To this end, we multiply and divide the integrand by $f(w_s)$ and use Lemma \ref{p1}. 
This motivates bounding the function
\begin{equation} \label{eq:b}
b(w):=\frac{ \textnormal{Tr}(\sigma(w)^\top\sigma(w))}{4f(w)}, \quad w \in \R^D,
\end{equation}
and we set
\begin{equation} \label{eq:B}
B_{\max}(r,w_0):=\sup_{w \in B_{r}(w_0), f(w) \neq 0} b(w).
\end{equation}

\textcolor{black}{
Finally, the last key step is to consider the stopped process 
$$\mathcal{E}_{t}=e^{c M_{t}-\frac12 c^2 \langle M \rangle_{t}}, \qquad c \in \R, \; t \geq 0.$$
This process is known as the exponential martingale, as justified by the following lemma, 
whose proof is deferred to Section \ref{sec:proofs}.
\begin{lemma} \label{mart}
Consider the {\sc sde} \textnormal{(\ref{w2})} initialized at some $w_0 \in \R^D$, and
suppose that Assumptions \ref{a1} and \ref{a1B} hold.
Then the process $(\mathcal{E}_{t})_{t \geq 0}$ is a nonnegative $\mathcal{F}_t$-martingale. 
\end{lemma}
As a consequence of Lemma \ref{mart}, and since $\mathcal{E}_0=1$, it follows that
for all $t \geq 0$,
\begin{equation} \label{eq:martingale}
\E[\mathcal{E}_{t}]=1.
\end{equation}}

\section{Convergence of the continuous-time {\sc sgd}}
\label{sec:main}

The following theorem provides sufficient conditions for convergence of the {\sc sde} (\ref{w2}) to a minimum of $f$, with positive probability, 
and also establishes an estimate for the rate of convergence.
More precisely, the theorem shows that if the process is initialized in a sufficiently small neighborhood of a global minimum
of $f$ and the noise parameter $\eta$ is sufficiently small, then the process converges to a minimum of $f$
with positive probability. 

We define the set of global minima of $f$ as
\begin{equation} \label{eq:s}
\mathcal{S}=\{w\in \R^D: f(w)=0\},
\end{equation}
which is non-empty by Assumption \ref{a1B}.

\begin{theorem} \label{t2}
Consider the {\sc sde} \textnormal{(\ref{w2})} initialized at some $w_0 \in \R^D$, and
suppose that Assumptions \ref{a1} and \ref{a1B} hold.
Assume that there exist $r>0$ and $\eta\ge 0$ such that
\begin{equation} \label{etay}
\eta < \frac{A_{\min}(r,w_0)}{G_{\max}(r,w_0)} ,
\end{equation}
(which is equivalent to $\theta(r,w_0,\eta)>0$),
\begin{equation} \label{bmax}
\eta B_{\max}(r,w_0)\leq \frac{1}{4},
\end{equation}
and
\begin{equation} \label{heta}
p:=\frac{2\sqrt{f(w_0)}}{r \sqrt{\theta(r,w_0,\eta)}}\left(1+\sqrt{\eta}\left(\frac{\sqrt{G_{\max}(r,w_0)}}{\sqrt{\theta(r,w_0,\eta)}}
+ \sqrt{B_{\max}(r,w_0)}\right)\right)<1.
\end{equation}
Then
\begin{equation} \label{e1}
\P(\tau_r \wedge \tau =\infty) \geq 1-p>0.
\end{equation}
Moreover, conditioned on the event $\{\tau_r \wedge \tau =\infty\}$,
the process $w_t$ converges almost surely to some $x^{\ast} \in  B_r(w_0)\cap \mathcal{S}$. 
Furthermore, for all $\epsilon>0$ and $t>0$,
\begin{equation} \label{rate}\begin{split}
\P \left( \Vert w_{t}-x^{\ast} \Vert >\epsilon \;\middle|\; \tau_r \wedge \tau =\infty\right) 
\leq \frac{r}{\epsilon} e^{-\theta(r,w_0,\eta) t/2}.
\end{split}\end{equation}
\end{theorem}

\begin{remark}{(On \cite{C22}.)}
When $\eta = 0$, Theorem~\ref{t2} reduces to the deterministic setting studied in \cite[Theorem 2.1]{C22}, which establishes convergence of (non-stochastic) gradient descent.
In this case, Assumption~\textnormal{(\ref{heta})} coincides with the condition introduced by Chatterjee, namely,
\begin{equation}
\label{eq:chat}
A_{\min}(r, w_0) > \frac{4f(w_0)}{r^2},
\end{equation}
and the convergence rate obtained in Theorem~\ref{t2} matches that of \cite{C22} when $\eta = 0$, namely exponential decay of the form
\[
\Vert w_{t}-x^{\ast} \Vert \leq r  e^{- \frac{A_{\min}(r, w_0)}{2} t}.
\]
\end{remark}

\begin{remark}{(On the Polyak–\L{}ojasiewicz (PL) condition.)}
Assumption \eqref{eq:chat} is closely related to the PL condition, which is widely used in non-convex optimization. 
The PL condition, together with Assumption~\ref{a1}, asserts that there exists a constant $\mu > 0$ such that for all $w \in \mathbb{R}^D$,
\begin{equation} \label{pl}
\|\nabla f(w)\|^2 \geq \mu f(w).
\end{equation}
Under this condition, and assuming that $\nabla f$ is globally Lipschitz continuous, \cite{KNS16} show that gradient descent with a suitable step size converges linearly to a global minimizer of $f$.
Assumption~\eqref{eq:chat} is clearly weaker than the PL condition: indeed, the PL inequality implies that $A_{\min}(r, w_0) \geq \mu$ for all centers $w_0$ and radii $r > 0$. Thus, the PL condition ensures that \eqref{eq:chat} holds for sufficiently large balls. By contrast, \eqref{eq:chat} only requires local boundedness, making it more broadly applicable than standard criteria for global convergence of gradient descent.
In this work, we extend condition~\eqref{eq:chat} to the stochastic setting, leading to Assumptions~\eqref{etay}, \eqref{bmax}, and \eqref{heta}. Notably, Assumption~\eqref{etay} is stronger than the PL condition, as it imposes a lower bound not only on $\|\nabla f(w)\|^2 / f(w)$, but on the smaller quantity
\[
\frac{\|\nabla f(w)\|^2}{f(w)} - \eta  \frac{\mathrm{Tr}(\sigma(w)^\top Hf(w) \sigma(w))}{2f(w)}.
\]
However, since $\eta$ can be chosen sufficiently small, it suffices to ensure that the term
\[
\frac{\mathrm{Tr}(\sigma(w)^\top Hf(w) \sigma(w))}{2f(w)}
\]
remains locally bounded. In Section~\ref{sec:app}, we demonstrate how this can be verified in the case of deep neural networks. See also Remark~\ref{re:bg} below.
\end{remark}
\begin{remark}\label{re:bg}
The additional conditions required in the stochastic setting involve the functions
\[
b(w) := \frac{\mathrm{Tr}(\sigma(w)^\top \sigma(w))}{4f(w)} \quad \text{and} \quad g(w) := \frac{\mathrm{Tr}(\sigma(w)^\top Hf(w) \sigma(w))}{2f(w)}
\]
defined in \textnormal{(\ref{eq:b})} and \textnormal{(\ref{eq:a})}, respectively.
In particular, we require that $B_{\max}(r, w_0) < \infty$ and $G_{\max}(r, w_0) < \infty$ for some radius $r > 0$ such that $B_r(w_0) \cap \mathcal{S} \neq \emptyset$.
To clarify the motivation for these assumptions, consider first the intuition behind the PL conditions~\eqref{eq:chat} and~\eqref{pl}. These conditions allow the gradient norm to decrease as $f(w)$ becomes small, but prevent it from vanishing too quickly; it must remain at least of order $\sqrt{f(w)}$.
Since $f$ is twice continuously differentiable, the entries of the Hessian matrix $Hf(w)$ are locally bounded on any ball $B_r(w_0)$. Consequently, boundedness of $g(w)$ implies that the growth of $\sigma(w)$ must also be controlled. Specifically, $\sigma(w)$ may grow, but at most proportionally to $\sqrt{f(w)}$. This is a natural assumption given the role of $\sigma(w)$ in the dynamics of the stochastic process.
In Section~\ref{sec:app}, we show that these conditions are plausible in the context of overparameterized neural networks.
\end{remark}

\begin{remark}{(On \( p \) as \( \eta \to 0 \).)}
Theorem \ref{t2} guarantees convergence to a global minimum of \( f \) with probability at least \( 1 - p \), where \( p \) remains bounded away from 1 for sufficiently small \( r \) and \( \eta \).
However, as \( \eta \to 0 \), the theorem does not guarantee that \( p \to 0 \) under condition~\eqref{eq:chat}. This apparent lack of continuity in \( p \) with respect to \( \eta \) may be an artifact of the proof technique.
It is natural to conjecture that \( p \to 0 \) as \( \eta \to 0 \). Supporting this, Section~\ref{sec:app} shows that, in the context of neural networks, the probability of convergence can be made arbitrarily close to 1 by choosing \( r \) and \( \eta \) sufficiently small.
\end{remark}

\begin{remark}{(On probability-one convergence.)}
Theorem \ref{t2} shows that if the stochastic process is initialized sufficiently close to a global minimum, and the functions \( f \) and \( \sigma \) satisfy certain regularity conditions, then convergence occurs with positive probability.
We conjecture that, in many cases, this positive-probability convergence implies a stronger property: from an \emph{arbitrary} initialization, the process converges \emph{almost surely} to a global minimum of \( f \).
This reasoning is based on the Markovian nature of the process. For convergence with positive probability, it suffices that there exists some time \( t \geq 0 \) and radius \( r > 0 \) such that the process enters the ball \( B_r(w_t) \) around some minimum \( w_t \in \mathcal{S} \), and that Assumptions~\eqref{etay}, \eqref{bmax}, and \eqref{heta} hold with \( w_0 \) replaced by \( w_t \).
Thus, the key point is that the process eventually reaches a sufficiently small neighborhood of the minima.
This is plausible if the gradient norm satisfies \( \|\nabla f(w)\| \to \infty \) as \( \|w\| \to \infty \), ensuring that the set of global minima \( \mathcal{S} \) is compact, and if the process exhibits diffusive behavior away from \( \mathcal{S} \). In particular, for any closed ball \( B \) that does not intersect \( \mathcal{S} \), the process almost surely does not remain in \( B \) indefinitely. This is reasonable given the noise structure encoded by \( \sigma(\cdot) \), which remains nondegenerate when \( f \) is bounded away from zero.
Establishing rigorous almost-sure convergence results from arbitrary initializations goes beyond the scope of this paper and is left for future work.
\end{remark}

\section{Related literature}
\label{sec:lit}

A significant effort has been devoted to the theoretical understanding
of the performance of gradient descent and stochastic gradient descent
algorithms in nonlinear optimization, with particular emphasis on
training neural networks. It is both natural and useful to study continuous-time approximations of these algorithms. 
For (non-stochastic) gradient descent this leads to the study of gradient flows.
The case when the objective function is convex is well understood \citep{Nes13}. 
While convexity is an important special case, the objective function in neural network training is typically nonconvex, which has motivated a large body of research. 

Our starting point is the result of \cite{C22}, who established a general sufficient condition 
for convergence of gradient descent. 
Chatterjee's criterion applies to deep neural networks with smooth activation functions, implying that 
gradient descent with appropriate initialization and step size converges to a global minimum of the loss function. 
We refer the reader to \cite{C22} for comparisons with earlier work on sufficient conditions for the convergence of
gradient descent. 
Our main result extends Chatterjee's result to a continuous-time approximation
of stochastic gradient descent under additional assumptions that are needed to accommodate the stochastic setting.

\cite{SeKaLeDeSr22} take a different approach to establish convergence properties of discrete-time stochastic
gradient descent by identifying general conditions under which stochastic gradient descent and 
gradient descent converge to the same point. 
In our analysis there is no reason why the two methods 
should converge to the same point, since we analyze the process \eqref{w2} directly.

As in \cite{C22}, we show that the sufficient conditions for stochastic gradient descent to converge to an optimum
are satisfied for a wide class of deep neural networks.
\cite{A23} also derive general sufficient conditions for the convergence of stochastic gradient descent. 
They write \eqref{w0} as 
\[
w_{k+1}=w_k-\eta \nabla f(w_k)+\eta  (\nabla f(w_k)-\nabla \ell(w_k, z_k))~,
\]
with the assumption
\[
\nabla f(w_k)-\nabla \ell(w_k, z_k)=\sqrt{\sigma f(w_k)} Z_{w_k, z_k},
\]
where $\sigma>0$ is a constant and $Z_{w_k, z_k}$ is a zero-mean noise term with identity covariance. 
In the continuous-time limit, this corresponds to
\[
\sigma(w)\sigma(w)^\top=\sigma^2 f(w) \E \left[ Z_t Z_t^\top \right].
\]
This assumption is different from ours and applies to a different class of problems. In particular, the 
simple overparametrized linear regression setup described in Section \ref{sec:app} does not satisfy this condition.

\cite{LZB22} establish in their Theorem 7 linear convergence of discrete stochastic gradient descent with random minibatches under the local PL condition. Their analysis is carried out in the empirical risk setting. Importantly, their results assume the local PL inequality as a hypothesis but no sufficient conditions are provided ensuring that it holds. Thus, while they prove that discrete-time SGD converges linearly once the local PL is satisfied, their framework does not address when PL holds.

\cite{NMM22} analyze the spectral properties of the neural tangent kernel (NTK) for deep ReLU networks, see Remark \ref{NTK} below. They obtain tight probabilistic bounds on the smallest eigenvalue of the NTK, showing that with high probability it is strictly positive when the network is sufficiently wide and randomly initialized. Their results apply to fully connected ReLU networks of arbitrary depth, under standard random initialization schemes, and they assume that the input distribution has sub-Gaussian tails to ensure concentration of the NTK spectrum. As shown in Remark \ref{NTK}, the PL constant can be identified with the smallest eigenvalue of the NTK. Thus, their result implies that a local PL inequality holds around initialization with high probability in the empirical risk setting. In contrast, our work establishes sufficient conditions for a PL-type inequality to hold at the population risk level, under the assumption of bounded input data (Assumption~\ref{a5}), allowing for a broader class of smooth activation functions beyond ReLU (Assumption~\ref{a7}), and for the continuous-time SDE approximation of SGD under these structural conditions.

\cite{LiGa21} study stochastic gradient Langevin dynamics similar to \eqref{w2} and their discretizations, but with
$\sigma(w_t)$ replaced by a constant.Their main contributions are to provide finite-time bounds on the generalization error and to derive error estimates for the discretization of Langevin dynamics, showing how closely discrete-time algorithms approximate the continuous-time diffusion in the empirical risk minimization setting.

\cite{ScPi24} analyze a continuous-time model of stochastic gradient descent similar to \eqref{w2}, but specialized to the case of linear regression. They show how least-squares SGD can be approximated by an SDE and investigate the resulting convergence and stability properties. Their analysis, however, is restricted to the well-specified linear model and does not extend to the overparameterized non-linear networks considered in our work.

A closely related line of research investigates the dynamics of stochastic gradient descent around saddle points, which frequently occur in high-dimensional nonconvex optimization.
It is well known that stochastic perturbations can help SGD escape saddle points efficiently. For instance, \cite{JGNKJ17} prove that, for discrete-time gradient descent with random perturbations, one can escape strict saddle points in polynomial time under standard smoothness assumptions. The more recent work \cite{ZLGU24} analyzes the behavior of discrete-time stochastic gradient descent near two different classes of saddle points and establishes conditions for their probabilistic stability, that is, when SGD is likely to converge to or escape from such saddles.
Our setting, however, is different: we assume $f \geq 0$ with $f(w)=0$ for some $w$. While saddle points may exist, the value of $f$ at such a point must be strictly positive, and our results guarantee that the dynamics given by
\eqref{w2} cannot get stuck in a saddle point.

From the point of view of stochastic differential equations, our result is also of interest in the context of explosion: in \eqref{w2} we only assume locally
Lipschitz coefficients, and hence the solution may blow up in finite time. Our results show that, despite this possibility, the process converges
with positive probability; see \cite{Mao1,Mao2}.

\section{Application to deep neural networks}
\label{sec:app}

In this section we show how Theorem \ref{t2} can be applied to the case of training multilayer neural networks using stochastic gradient descent.
To this end, we verify the conditions of the theorem for this particular setting. 

\textcolor{black}{Consider a multilayer feedforward neural network defined as follows. 
The weights of the network are given by $(W_1, W_2,\ldots,W_L)$, where each $W_{\ell}$ is a $d_{\ell} \times d_{\ell-1}$ matrix,
with $d_{L}=1$ and $d_0=d$. 
The layer $W_1$ is called the \emph{output layer}, while $W_2,\ldots,W_{L-1}$ are the \emph{hidden layers}. 
The number of layers $L \geq 2$ is the \emph{depth} of the network, and the maximum of $d_1,\ldots,d_L$ is the \emph{width}.  
We also consider a sequence of bias vectors $b_1,\ldots, b_L$, with $b_{\ell} \in \R^{d_{\ell}}$, and fixed activation functions $v_1,\ldots,v_{L}: \R  \rightarrow \R$, where $v_L$ is the identity map.}

\textcolor{black}{The parameter vector is 
\[
w=(W_1, b_1, \ldots, W_L, b_L)\in \R^D, 
\qquad D=\sum_{\ell=1}^L d_{\ell}(d_{\ell}+1).
\]
Given $w \in \R^D$, the network defines the map $\beta(w,\cdot):\R^d \rightarrow \R$ by
\[
\beta(w,x)=v_L\!\Big(W_L v_{L-1}\!\big(\cdots W_2 v_1(W_1x+b_1)+b_2 \cdots\big)+b_L\Big),
\]
where each activation function $v_\ell$ acts componentwise on vectors of dimension $d_\ell$, and satisfy the following condition.  
\begin{assumption} \label{a7}
The activation functions $v_1,\ldots,v_L$ satisfy $v_\ell \in \mathcal{C}^2(\R)$, $v_{\ell}(0)=0$, and $v'_{\ell}(y)>0$ for each $\ell \in \{1,\ldots,L\}$ and all $y \in \R$.
\end{assumption}}

With quadratic loss, the learning problem consists of minimizing
\[
f(w)=\E\!\left[(\beta(w,X) -Y)^2 \right],
\]
where the random pair $(X,Y)$ takes values in $\R^d\times \R$.  
Let $\Sigma_X=\E[XX^\top]$. We assume the following.
\begin{assumption} \label{a4}
$\lambda_{\min}(\Sigma_X)>0$.
\end{assumption}
\begin{assumption} \label{a5}
There exists $K>0$ such that $\Vert X \Vert \leq K$ almost surely.
\end{assumption}

In order to apply Theorem \ref{t2}, we need to assume $f(w)=0$ for some $w\in \R^D$ (Assumption \ref{a1B}), which is equivalent to the following.
\begin{assumption} \label{a3}
There exists $w^{\ast}=(W_1^{\ast}, b_1^{\ast}, \ldots, W_L^{\ast}, b_L^{\ast})\in \R^D$ such that
\[
Y=\beta(w^{\ast},X).
\]
\end{assumption}
Under Assumption \ref{a3}, the loss can be written as
\[
\ell(w, X)=(\beta(w,X)-\beta(w^{\ast},X))^2,
\]
and the set of global minima of $f$ defined in \eqref{eq:s} is the (non-empty) closed subset
\begin{equation} \label{snet}
\mathcal{S}=\{ w=(W_1, b_1, \ldots, W_L, b_L) \in \R^D: \beta(w,X)=\beta(w^{\ast},X) \text{ a.s.}\}.
\end{equation}

Let $w_t$ be the solution to the {\sc sde} \eqref{w2} associated with this minimization problem. We assume the following structure for the initialization.
\begin{assumption} \label{a6}
The initial condition  $w_0=(W^0_1, b^0_1, \ldots, W_L^0, b_L^0)\in \R^D$ of the {\sc sde} \eqref{w2} satisfies:  
$W^0_1=0$, $b^0_{\ell}=0$ for each $\ell \in \{1,\ldots,L\}$, and all entries of $W_2^0,\ldots,W_{L}^0$ are nonnegative.
\end{assumption}

The following theorem provides sufficient conditions for convergence of the {\sc sde} \eqref{w2} associated with this problem as an application of Theorem \ref{t2}.
\begin{theorem} \label{nn1}
Consider the minimization problem associated with the neural network above, with activation functions satisfying Assumption \ref{a7} and $(X,Y)$ satisfying Assumptions  \ref{a4}, \ref{a5}, and \ref{a3}. 
Let $w_t$ be the solution to the {\sc sde} \eqref{w2} with initial condition satisfying Assumption \ref{a6}.
Let $\gamma>0$ be the minimum entry of $W_2^0,\ldots,W_{L-1}^0$, and let $M>0$ be the maximum entry of $W_2^0,\ldots,W_{L}^0$. 
Then, for all $\delta \in (0,1)$ there exist constants $N>0$ and $\eta_0>0$ depending only on $\lambda_{\min}(\Sigma_X)$, $\gamma$, $M$, $K$, and $v_1,\ldots,v_L$, such that if the entries of $W_L^0$ are all $\geq N$ and  $\eta \leq \eta_0$, then
\[
\P(\tau_{\gamma/2} \wedge \tau=\infty)\geq 1-\delta.
\]
Moreover, conditioned on this event, $w_t$ converges almost surely to some element $x^{\ast}$ in $B_{\gamma/2}(w_0)\cap \mathcal{S}$, and for all 
$\epsilon>0$ and $t>0$,
\begin{equation} \label{ratebis}\begin{split}
\P \left( \Vert w_{t}-x^{\ast} \Vert >\epsilon \,\middle|\, \tau_{\gamma/2} \wedge \tau =\infty\right) 
\leq \frac{\gamma}{2\epsilon} e^{-C\lambda_{\min}(\Sigma_X) (N-\gamma/2) t},
\end{split}\end{equation}
where $C>0$ depends only on $\gamma$, $M$, $K$, and $v_1,\ldots,v_L$.
\end{theorem}

The intuition behind Theorem~\ref{nn1} is that if the entries of the final layer $W_L^0$ are chosen large enough and the step size $\eta$ is sufficiently small, then the deterministic drift term $-\nabla f(w)$ dominates the stochastic noise. 
In this regime, the dynamics of \eqref{w2} are effectively pushed toward the global minimum set $\mathcal{S}$, which ensures convergence with high probability.  
The proof is divided into two steps.  
In the first step, we establish bounds on the functions $a(w)$, $b(w)$, and $g(w)$ defined in \eqref{eq:a} and \eqref{eq:b}. 
This is done in Lemma~\ref{prelem} below, where only Assumptions~\ref{a7}, \ref{a5}, and \ref{a3} are needed.  
In the second step, we apply Theorem~\ref{t2}, which
requires verifying that its assumptions hold in our setting, which in turn are implied by Assumptions~\ref{a4} and \ref{a6}. See the remarks below for a more detailed discussion of these assumptions.
\begin{remark}{(On the PL condition.)} \label{NTK}
In the overparameterized regime, the PL condition is naturally linked to the spectral properties of the Neural Tangent Kernel (NTK). See for instance \cite{LZB22} and the references therein.
We consider a neural network \( \beta(w, X_i) \in \mathbb{R} \), and define the empirical loss over i.i.d.\ training data \( (X_1, Y_1), \dots, (X_n, Y_n) \) as
\begin{equation} \label{emp_loss}
f_n(w) := \sum_{i=1}^n h_i(w)^2,
\end{equation}
where $h_i(w):=\beta(w, X_i) - Y_i$. 
We set $h(w)$ to be the column vector whose entries are $(h_1(w),\ldots h_n(w))$.
We consider the $n \times d$ Jacobian matrix $J(w)$ whose $i$th-row is the vector
$(\nabla \beta(w, X_i))^{\top}$.
Then, the (empirical) NTK is the $n \times n $ matrix defined as 
\begin{equation} \label{emp_loss2}
N(w)=J(w) (J(w))^{\top}.
\end{equation}
Because $N(w)$ is a Gram matrix, it is symmetric positive semidefinite, and for any $\xi \in \R^n$,
$$
\xi^{\top} N(w) \xi \geq \lambda_{\min}(N(w)) \sum_{i=1}^n \xi_i^2.
$$
Therefore, since $\nabla f_n(w)=2 (J(w))^{\top} h(w)$, we get that
\begin{equation} \label{PL2} \begin{split}
\|\nabla f_n(w)\|^2 &=(\nabla f_n(w))^{\top} \nabla f_n(w)
=  4 (h(w))^{\top} N(w) h(w)\\
& \geq 4 \lambda_{\min}(N(w))  f_n(w).
\end{split}
\end{equation}
Thus, in this case, the PL constant in \textnormal{(\ref{pl})} is given by \( \mu = 4\lambda_{\min}(N(w)) \). In sufficiently overparameterized networks, as for instance in \cite{C22}, under mild assumptions on the initialization and the data distribution, the smallest eigenvalue \( \lambda_{\min}(N(w)) \) is strictly positive in a neighborhood of the initialization. This explains why, in those settings, gradient-based methods can achieve fast convergence despite the nonconvexity of the loss landscape. 
Most existing results establish such PL-type inequalities for the \emph{empirical loss} \(f_n\). See for instance \cite{NMM22} and the references therein.
In general, these do not automatically extend to the \emph{population risk} considered in this paper
\[
f(w) = \mathbb{E}[(\beta(w, X) - Y)^2],
\]
unless additional assumptions are imposed. 
A distinctive feature of our work is that we establish convergence directly for the population risk \(f\), by imposing Assumptions \textnormal{\ref{a7}–\ref{a6}}. 
\end{remark}
\begin{remark}{(On Assumption \textnormal{\ref{a7}}.)}
The class of activation functions allowed by Assumption~\ref{a7} includes many commonly used functions, such as the linear activation \( v(y) = y \), the bipolar sigmoid \( v(y) = \tfrac{1 - e^{-y}}{1 + e^{-y}} \), and the hyperbolic tangent \( v(y) = \tanh(y) \).
The condition \( v(0) = 0 \) is not essential and can be relaxed by incorporating bias terms. 
As a result, Theorem~\ref{nn1} also applies to other widely used activations such as the sigmoid \( v(y) = \tfrac{1}{1 + e^{-y}} \), the softplus (smoothed ReLU) \( v(y) = \log(1 + e^y) \), and the complementary log-log function \( v(y) = 1 - e^{-e^{-y}} \).
However, the requirement that \( v \) be twice differentiable excludes non-smooth activations such as ReLU \( v(y) = \max\{y, 0\} \), the step function \( v(y) = \mathbf{1}_{\{y>0\}} \), and other piecewise linear functions.
\end{remark}
\begin{remark}{(On Assumption \ref{a6}.)}
Theorem~\ref{nn1} identifies a set of initializations \( w_0 \in \mathbb{R}^D \) for which the probability that the solution to the stochastic differential equation~\eqref{w2} converges can be made arbitrarily close to one.
The key idea is to choose \( w_0 \) such that two natural conditions are met:
 The neural network 
 satisfies \( \beta(w_0, X) = 0 \), leading to a bounded loss \( f(w_0) \);
The PL condition holds locally around \( w_0 \). These two conditions allow us to bound $p$ defined in (\ref{heta}) for sufficiently small $\eta$. 
Observe that any initialization satisfying these conditions—bounded initial loss and local PL property—can be used to guarantee high-probability convergence under SGD.
\end{remark}

\section{Proofs}
\label{sec:proofs}

\begin{proof}[Proof of Lemma \ref{l1}]
If $f(w_0)=0$, since $f$ is a nonnegative $\mathcal{C}^2(\R^D)$ function, we have $\nabla f(w_0)=0$. 
Moreover, since $f(w_0)=\E[\ell(w_0,x)]$ and $\ell$ is nonnegative, we get that $\ell(w_0,Z)=0$ and thus $\nabla \ell(w_0,Z)=0$, as it is a $\mathcal{C}^2(\R^D)$ function in its first variable.
In particular, $\sigma(w_0)=0$. Then $w_t=w_0$ for all $t>0$, and the statement is true for $t=0$. 
If $w_s=x$ for some $s>0$ such that $f(x)=0$, since the process $w_t$ is time-homogeneous, the distribution of $w_t$ starting at $w_s=x$ is the same as the distribution of $w_{t-s}$ starting at $w_0=x$. By the argument above we conclude that $w_{t-s}=x$ for all $t>s$, which completes the proof.
\end{proof}

\begin{proof}[Proof of Lemma \ref{p1}]
\textcolor{black}{As explained in Section \ref{sec:pre}, we apply the multi-dimensional It\^o formula (Theorem \ref{itofor}) to the function $\log f(w_{t\wedge \tau_r \wedge \tau})$. That is, we consider the process $x_t=w_{t\wedge \tau_r \wedge \tau}$
and the function $V=\log f$. In particular, 
$u_t=-\nabla f(w_{t\wedge \tau_r \wedge \tau})$ and $v_t=\sqrt{\eta} \sigma(w_{t\wedge \tau_r \wedge \tau})$.
Observe that, by adding and subtracting the term $\nabla f(w_0)$ and using Assumption \ref{a1}, we get
$$
\int_0^\top \Vert u_t \Vert dt \leq 
(\textnormal{Lip}(f, r,w_0)r +\Vert \nabla f(w_0) \Vert)(\tau_r \wedge \tau),
$$
where Lip denotes the Lipschitz constant defined in (\ref{lip}).
Hence, if $\tau_r \wedge \tau <\infty$ a.s., assumption (\ref{h1}) of Theorem \ref{itofor} holds.
On the other hand, if $\tau_r \wedge \tau=\infty$ a.s., then $T=\infty$ and assumption (\ref{h1}) also holds since the process exists for all times and remains inside the ball. Proceeding similarly, we can easily show that Assumption \ref{a1} 
implies condition (\ref{h2}) of Theorem \ref{itofor}.
Finally, since the loss function $\ell$ is assumed to be twice differentiable in the first variable, we conclude that all the assumptions of Theorem \ref{itofor} are satisfied. Thus, using (\ref{md}), we obtain}
\begin{equation*} \begin{split}
   \log f(w_{t \wedge \tau_r \wedge \tau})
   &=\log f(w_0)-\int_0^{t \wedge \tau_r \wedge \tau} (a (w_s)-\eta g(w_s))  ds +M_{t}-\frac{1}{2} \langle M \rangle_{t }\\
   &\leq \log f(w_0)-(t \wedge \tau_r \wedge \tau) \theta(r,w_0,\eta) +M_{t}-\frac{1}{2} \langle M \rangle_{t}.
\end{split}
\end{equation*}
Then, taking exponentials, the result follows.
\end{proof}

\begin{proof}[Proof of Lemma \ref{t1}]
Assume that $\tau_r \wedge \tau =\infty$. Observe that, in particular,  $T=\infty$.
Then, by the same arguments as in the proof of Lemma \ref{p1}, we have that for all $t>0$ a.s.
\begin{equation} \label{itof2}
\frac{\log f(w_{t })}{t} \leq \frac{\log f(w_0)}{t} -\theta(r,w_0,\eta)  + \frac{1}{t}\left( M_{t}-\frac{1}{2} \langle M \rangle_{t}\right)~.
\end{equation}
On the other hand, appealing to the exponential martingale inequality (see \cite[Theorem 7.4, page 44]{Mao1}), we get that
for any fixed $n>0$ and for all $x>0$,
\begin{equation*} 
\P\left\{ \sup_{t\in [0,n]} \left( M_{t}-\frac{1}{2} \langle M \rangle_{t} \right)  > x \right\} \leq e^{-x}~.
\end{equation*}
\textcolor{black}{
Choosing $x=2 \log n$ and appealing to the Borel–Cantelli lemma, we get that for almost all $\omega \in \Omega$, there exists an integer $n_0=n_0(\omega)>1$ such that for all $t \in [0,n]$ and $n \geq n_0$,
$$
M_t -\frac{1}{2}
\langle M \rangle_t\leq 2 \log n.
$$
Therefore, by (\ref{itof2}), we obtain that for all $t \in [n-1,n]$ and $n \geq n_0$,
\begin{equation*} 
\frac{\log f(w_{t })}{t} \leq \frac{\log f(w_0)}{t} -\theta(r,w_0,\eta)  + \frac{2 \log n}{n-1} \quad \text{a.s.}
\end{equation*}
It follows that 
$$
\limsup_{t \rightarrow \infty} \frac{\log f(w_{t })}{t} \leq -\theta(r,w_0,\eta) \quad \text{a.s.},
$$
which concludes the proof.}
\end{proof}

\textcolor{black}{
\begin{proof}[Proof of Lemma \ref{mart}]
We apply the multi-dimensional It\^o formula (Theorem \ref{itofor}) to the one-dimensional process $\mathcal{E}_t$ with $x_t=c M_{t}-\frac12 c^2 \langle M \rangle_{t}$ and $V(x)=e^x$. We get
\begin{equation*}
\begin{split}
\mathcal{E}_t=1+ 
  \sqrt{\eta} c \int_0^{t \wedge \tau_r \wedge \tau} \mathcal{E}_s \frac{(\nabla f(w_s))^\top \sigma(w_s)}{f(w_s)}   dB_s,
\end{split}
\end{equation*}
which is an $\mathcal{F}_t$-martingale since an $\mathcal{F}_t$-stopped It\^o integral is an $\mathcal{F}_t$-martingale (see \cite[Theorem 3.3, page 11]{Mao1}).
\end{proof}}

\begin{proof}[Proof of Theorem \ref{t2}]
Let $0\leq u<t$ and set $\bar{t} := t \wedge\tau_r \wedge \tau$ and $\bar{u} := u \wedge\tau_r \wedge \tau$.
Let $\epsilon>0$. 
Then, by Markov's inequality,
\begin{equation} \label{step1} \begin{split}
&\P \left( \Vert w_{\bar{t}}-w_{\bar{u}} \Vert >\epsilon  \right)  
\leq \frac{\E\left[ \Vert w_{\bar{t}}-w_{\bar{u}} \Vert  \right]}{\epsilon}\\
&\qquad  \leq \frac{\E\left[  \int_{\bar{u}}^{\bar{t}} \Vert \nabla f(w_s) \Vert ds  \right]}{\epsilon}
+\sqrt{\eta} \frac{\E\left[ \Vert \int_{\bar{u}}^{\bar{t}} \sigma(w_s) dB_s\Vert \right]}{\epsilon}\\
& \qquad \leq \frac{\E\left[  \int_{\bar{u}}^{\bar{t}} \Vert \nabla f(w_s) \Vert ds  \right]}{\epsilon}
+\sqrt{\eta} \frac{\left\{\E\left[  \int_{\bar{u}}^{\bar{t}} \text{Tr}((\sigma(w_s))^\top\sigma(w_s)) ds  \right]\right\}^{1/2}}{\epsilon}~,
\end{split}\end{equation}
where the last inequality follows from the Cauchy–Schwarz inequality and 
\cite[Theorem 5.21 page 28]{Mao1}.

By the Cauchy–Schwarz inequality,
\begin{equation}
\label{eq:cauchyschwarz}
\E\left[  \int_{\bar{u}}^{\bar{t}} \Vert \nabla f(w_s) \Vert ds  \right]
\leq \left\{\E\left[  \int_{\bar{u}}^{\bar{t}} \frac{\Vert \nabla f(w_s) \Vert^2}{2 \sqrt{f(w_s)}} ds  \right] \right\}^{1/2} \left\{\E\left[  \int_{\bar{u}}^{\bar{t}} 2\sqrt{f(w_s)} ds  \right]\right\}^{1/2}.
\end{equation}

Observe that, on the event $\{u\geq \tau_r \wedge \tau\}$, we have $\bar{u}=\bar{t}$ and then all the integrals between $\bar{u}$ and $\bar{t}$ vanish.
Thus, it suffices to consider the event $A:= \{u<\tau_r \wedge \tau\}$, so $\bar{u}=u$.

\textcolor{black}{We first bound the second term on the right-hand side of \eqref{eq:cauchyschwarz}.} Using Lemma \ref{p1}, we get
\begin{equation} \label{step2} \begin{split}
\E\left[  \int_{\bar{u}}^{\bar{t}} 2\sqrt{f(w_s)} ds {\bf 1}_A  \right] 
&\leq \E\left[  \int_u^{\bar{t}} 2 \sqrt{f(w_0)} e^{-\theta(r,w_0,\eta) s/2} e^{\frac12 M_{s}-\frac14  \langle M \rangle_{s}}ds {\bf 1}_A \right]\\
&\leq   \int_u^{\infty} 2 \sqrt{f(w_0)} e^{-\theta(r,w_0,\eta) s/2} \E\left[e^{\frac12 M_{s}-\frac18  \langle M \rangle_{s}}\right] ds  \\
&= \frac{4 \textcolor{black}{\sqrt{f(w_0)}}}{\theta(r,w_0,\eta)}e^{-\theta(r,w_0,\eta) u/2},
\end{split}
\end{equation}
where in the second inequality we used that $\langle M \rangle_{s}\ge 0$ for all $s\ge 0$, and
the equality follows from (\ref{eq:martingale}) with $c=\frac12$.

To bound the first term on the right-hand side of \eqref{eq:cauchyschwarz},
we apply the multi-dimensional It\^o formula (Theorem \ref{itofor}) to $\sqrt{f(w_{\bar{t}})}$. That is,
\begin{equation*}
\begin{split}
\sqrt{f(w_{\bar{t}})}=&\sqrt{f(w_{\bar{u}})}-\int_{\bar{u}}^{\bar{t}}  \frac{\Vert \nabla f(w_s)\Vert^2}{2 \sqrt{f(w_s)}} ds+ \sqrt{\eta}\int_{\bar{u}}^{\bar{t}} \frac{(\nabla f(w_s))^\top}{2\sqrt{f(w_s)}}  \sigma(w_s) dB_s+Z_{\bar{t}},
\end{split}
\end{equation*}
where
\begin{equation*}
\begin{split}
Z_{\bar{t}}:=\frac{\eta}{2}\int_{\bar{u}}^{\bar{t}} \text{Tr}\left((\sigma(w_s))^\top \left(\frac{ Hf(w_s) }{2\sqrt{f(w_s)}}- \frac{\nabla f(w_s) 
(\nabla f(w_s))^\top}{4 f(w_s)^{3/2}}\right)\sigma(w_s)\right)ds~.
\end{split}
\end{equation*}
Taking expectations, and noting that the stochastic integral term has zero mean, we get 
\begin{equation*} 
\begin{split}
\E\left[\int_{\bar{u}}^{\bar{t}}   \frac{\Vert \nabla f(w_s)\Vert^2}{2 \sqrt{f(w_s)}} ds\right]=\E\left[\sqrt{f(w_{\bar{u}})}\right]-\E\left[ \sqrt{f(w_{\bar{t}})}\right]+\E\left[Z_{\bar{t}}\right].
\end{split}
\end{equation*}
Then, 
by the definition of $G_{\max}(r,w_0)$ and
using a similar argument as above with Lemma \ref{p1} and  (\ref{eq:martingale}), we obtain
\begin{equation} \label{step3} \begin{split}
\E \left[\int_{\bar{u}}^{\bar{t}}  \frac{\Vert \nabla f(w_s)\Vert^2}{2 \sqrt{f(w_s)}} ds {\bf 1}_A\right]&\leq \E\left[\sqrt{f(w_u)}{\bf 1}_A\right]+\frac{\eta}{2}\E \left[\int_u^{\bar{t}} \frac{\text{Tr}\left((\sigma(w_s))^\top Hf(w_s) \sigma(w_s)\right)}{2\sqrt{f(w_s)}}ds {\bf 1}_A\right]\\
&\leq \E\left[\sqrt{f(w_u)}{\bf 1}_A\right]+\frac{\eta G_{\max}(r,w_0)}{2}  \E\left[  \int_u^{\bar{t}} \sqrt{f(w_s)} ds  {\bf 1}_A\right] \\
&\leq \sqrt{f(w_0)}e^{-\theta(r,w_0,\eta) u/2}+\frac{\eta G_{\max}(r,w_0)}{\theta(r,w_0,\eta)} \sqrt{f(w_0)}e^{-\theta(r,w_0,\eta) u/2}.
\end{split}
\end{equation}
\textcolor{black}{Substituting equations (\ref{step2}) and (\ref{step3}) into (\ref{eq:cauchyschwarz}) yields}
\begin{equation} \label{step4}
\E\left[  \int_{\bar{u}}^{\bar{t}} \Vert \nabla f(w_s) \Vert ds  \right]\leq \frac{2\sqrt{f(w_0)}}{\sqrt{\theta(r,w_0,\eta)}}e^{-\theta(r,w_0,\eta) u/2}\left( 1+\frac{\sqrt{\eta G_{\max}(r,w_0)}}{\sqrt{\theta(r,w_0,\eta)}}\right)~.
\end{equation}
Moreover,
by the definition of $B_{\max}(r,w_0)$ and
appealing to Lemma \ref{p1}, we get
\begin{equation} \label{step5} \begin{split}
&\E\left[  \int_{\bar{u}}^{\bar{t}} \text{Tr}\left((\sigma(w_s))^\top\sigma(w_s)\right) ds {\bf 1}_A \right]\leq
B_{\max}(r,w_0)  \E\left[  \int_u^{\bar{t}} 4 f(w_s) ds {\bf 1}_A \right]\\
&\qquad \qquad\leq B_{\max}(r,w_0)  \E\left[  \int_u^{\infty} 4 f(w_0) e^{-\theta(r,w_0,\eta) s} e^{ M_s-\frac12  \langle M \rangle_s}ds  \right]\\
&\qquad \qquad =  B_{\max}(r,w_0) \frac{4 f(w_0)}{\theta(r,w_0,\eta)} e^{-\theta(r,w_0,\eta) u},
\end{split}
\end{equation}
where the last equality follows again from (\ref{eq:martingale}).

\textcolor{black}{Substituting equations (\ref{step4}) and (\ref{step5}) into (\ref{step1}) shows} that for all $0\leq u<t$ and $\epsilon>0$,
\begin{equation} \label{limit}\begin{split}
&\P \left( \Vert w_{t \wedge \tau_r \wedge \tau}-w_{u \wedge \tau_r \wedge \tau} \Vert >\epsilon \right) \\
&\leq \frac{2\sqrt{f(w_0)}}{\epsilon \sqrt{\theta(r,w_0,\eta)}}e^{-\theta(r,w_0,\eta) u/2}\left(1+\sqrt{\eta}\left(\frac{\sqrt{G_{\max}(r,w_0)}}{\sqrt{\theta(r,w_0,\eta)}}+ \sqrt{ B_{\max}(r,w_0) }\right)\right),
\end{split}\end{equation}
where we observe that the right-hand side is independent of $t$, $\tau_r$, and $\tau$.

We are now ready to prove the two statements of the theorem. We begin with 
(\ref{e1}). Taking $u=0$, $t \uparrow \tau_r \wedge \tau$, and $\epsilon=r$ in (\ref{limit}), we get
\begin{equation*} \begin{split}
\P\left( \tau_r \wedge \tau <\infty\right)\leq \frac{2\sqrt{f(w_0)}}{r \sqrt{\theta(r,w_0,\eta)}}\left(1+\sqrt{\eta}\left(\frac{\sqrt{G_{\max}(r,w_0)}}{\sqrt{\theta(r,w_0,\eta)}}+ \sqrt{ B_{\max}(r,w_0) }\right)\right):= p<1.
\end{split}\end{equation*}
This implies
\begin{eqnarray*}
\P\left( \tau_r \wedge \tau =\infty\right) \geq 1-p>0~,
\end{eqnarray*}
proving (\ref{e1}).

We next prove the second statement of the theorem. 
Using (\ref{limit}) and condition (\ref{heta}), we obtain that for all $0\leq u<t$ and $\epsilon>0$,
\begin{equation} \label{limit2}\begin{split}
\P \left( \Vert w_{t \wedge \tau_r \wedge \tau}-w_{u \wedge \tau_r \wedge \tau} \Vert >\epsilon \right) \leq \frac{r}{\epsilon} e^{-\theta(r,w_0,\eta) u/2}.
\end{split}\end{equation}
Assume that $\tau_r \wedge \tau=\infty$. Then (\ref{limit2}) shows that $w_t$ is a Cauchy sequence in probability. Therefore, by \cite[Theorem 3, Chapter 6]{boro}, the sequence $w_t$ converges in probability to some $x^{\ast} \in B_r(w_0)$ as $t \rightarrow \infty$.
Moreover, taking $t \uparrow\infty$ in (\ref{limit2}), we obtain (\ref{rate}). Since the rate of convergence is exponential, we conclude that $w_t$ converges to $x^{\ast}$ almost surely, conditioned on the event $\{\tau_r \wedge \tau=\infty\}$. Finally, by Lemma \ref{t1}, we have $x^{\ast} \in \mathcal{S}$. This concludes the proof.
\end{proof}

\textcolor{black}{
Next we turn to the proof of Theorem \ref{nn1}.
We start with a preliminary lemma that gives bounds for the functions $a(w)$, $b(w)$, and $g(w)$ defined in (\ref{eq:a}) and (\ref{eq:b}), which will be useful in order to bound the functions (\ref{eq:A}) and (\ref{eq:B}). Recall the setup and notation introduced in Section \ref{sec:app}. To state the lemma, we first introduce some notation to study the derivative of the neural network with respect to the input layer $W_1$, following \cite{C22}.}

\textcolor{black}{
Given $w=(W_1, b_1, \ldots, W_L, b_L)\in \R^D$, we recursively define $\beta_1(w,x)=v_1(W_1x+b_1)$, and for $2 \leq \ell \leq L$, 
$$\beta_{\ell}(w,x)=v_{\ell}(W_{\ell} v_{\ell-1}(\cdots W_2 v_1(W_1x+b_1)+b_2 \cdots)+b_{\ell}),$$
so that $\beta=\beta_L$. Note that $\beta_{\ell}(w,\cdot):\R^d \rightarrow \R^{d_{\ell}}$.
Define $g_1(w,x)=W_1x+b_1$ and for $2 \leq \ell \leq L$,
$$
g_{\ell}(w,x)=W_{\ell} \beta_{\ell-1}(w,x)+b_{\ell},
$$
so that $\beta_{\ell}(w,x)=v_{\ell}(g_{\ell}(w,x))$. We denote by $D_{\ell}(w,x)$ the $d_{\ell} \times d_{\ell}$ diagonal matrix whose diagonal consists of the elements of the vector $v'_{\ell}(g_{\ell}(w,x))$. Then, as noted in \cite{C22}, the partial derivative of $\beta$ with respect to the $(i,j)$ component of $W_1$ is
$$
\partial_{i,j} \beta(w,x)=x_j q_i(w,x),
$$
where
\begin{equation} \label{eq:q}
q_i(w,x)= W_L D_{L-1}(w,x) W_{L-1} \cdots W_2 D_1(w,x) e_i,
\end{equation}
where $e_i \in \R^{d_1}$ is the vector whose $i$th component is $1$ and the rest are zero.}

\textcolor{black}{
Using this notation, we have the following result.
\begin{lemma} \label{prelem}
Consider Assumptions \textnormal{\ref{a7}, \ref{a5},} and \textnormal{\ref{a3}}. Let $a(w)$, $b(w)$, and $g(w)$ be the functions defined in \textnormal{(\ref{eq:a})} and \textnormal{(\ref{eq:b})}.
Then, 
for all $w=(W_1, b_1, \ldots, W_L, b_L)\in \R^D$ such that $f(w) \neq 0$, 
\begin{equation} \begin{split} \label{lambda}
a(w) \geq  4 \lambda_{\min}(\Sigma_X) \sum_{i=1}^{d_1} \min_{x \in \R^d: \Vert x \Vert \leq K} (q_i(w,x))^2 ,
\end{split}
\end{equation}
\begin{equation} \label{prelem2}
b(w)\leq \max_{x \in \R^d: \Vert x \Vert \leq K} \Vert \nabla \beta(w,x)\Vert^2,
\end{equation}
and
\begin{equation} \label{prelem3} \begin{split}
g(w) &\leq 16\max_{x \in \R^d: \Vert x \Vert \leq K} \bigg(\Vert \nabla \beta(w,x)\Vert^2\big(\Vert \nabla \beta(w,x)\Vert^2 \\
&\qquad +D \vert \beta(w,x)-\beta(w^{\ast},x)\vert \lambda_{\max}(H (\beta(w,x))\big) \bigg).
\end{split}
\end{equation}
\end{lemma}}

\begin{proof}\textcolor{black}{
Let $w=(W_1, b_1, \ldots, W_L, b_L)\in \R^D$ be fixed with $f(w) \neq 0$.
Recall that
\begin{equation*} 
f(w)=\E\left[ (\beta(w,X)-\beta(w^{\ast},X))^2 \right].
\end{equation*}}

\textcolor{black}{
We first prove (\ref{lambda}).
Let $X_1,\ldots,X_n$ be $n$ independent copies of the random vector $X$.
Observe that, by Assumption \ref{a3}, for all $i \in \{1,\ldots, n\}$, $\Vert X_i \Vert \leq K$ a.s. We consider the empirical loss as defined in (\ref{emp_loss}), that is,
\begin{equation*} 
f_n(w) := \sum_{i=1}^n (\beta(w,X_i) - \beta(w^{\ast},X_i))^2.
\end{equation*}
Then, using inequality (\ref{PL2}), we get that
\begin{equation} \label{eq:ine1}
\Vert \nabla f_n(w)  \Vert^2
 \geq 4\lambda_{\text{min}} \left( N(w)\right) f_n(w),
\end{equation}
where the matrix $N(w)$ is defined in (\ref{emp_loss2}).
Appealing to inequality (5.4) in the proof of \cite[Theorem 4.1]{C22}, we get that
\begin{equation}   \label{eq:ine2}
\lambda_{\text{min}} \left( N(w)\right)\geq  n\lambda_{\text{min}} \left(\frac{1}{n}\chi^\top \chi \right) \sum_{i=1}^{d_1} \min_{x \in \R^d: \Vert x \Vert \leq K} (q_i(w,x))^2,
\end{equation}
where $\chi$ is the $d \times n$ matrix whose columns are the vectors $X_1,\ldots, X_n$ and $q_i(w,x)$ is defined in (\ref{eq:q}). In order to obtain (\ref{eq:ine2}) it suffices to consider the terms that correspond to the derivative with respect to $W_1$ and lower bound all the other derivatives by zero.}

\textcolor{black}{
Therefore, from (\ref{eq:ine1}) and (\ref{eq:ine2}), we conclude that
$$
\frac{1}{n}\frac{\Vert \nabla f_n(w) \Vert^2}{f_n(w)}
 \geq 4 \lambda_{\text{min}} \left(\frac{1}{n}\chi^\top \chi \right) \sum_{i=1}^{d_1} \min_{x \in \R^d: \Vert x \Vert \leq K} (q_i(w,x))^2.
$$
Now, by the law of large numbers, as $n \rightarrow \infty$, $\frac{1}{n}\frac{\Vert \nabla f_n(w) \Vert^2}{f_n(w)}$ and $\lambda_{\text{min}} \left(\frac{1}{n}\chi^\top \chi \right)$
converge almost surely to $\frac{\Vert \nabla f(w) \Vert^2}{f(w)}=a(w)$ and $\lambda_{\text{min}}(\Sigma_X)$, respectively. This proves inequality (\ref{lambda}).}

\textcolor{black}{
To prove (\ref{prelem2}), observe that
\begin{equation*} \begin{split} 
\textnormal{Tr}\left((\sigma(w))^\top\sigma(w)\right)&=\textnormal{Tr}\left(\sigma(w)(\sigma(w))^\top \right)\\
&=\E \left[\Vert \nabla\ell(w,X) \Vert^2\right]-\Vert \nabla f(w) \Vert^2\\
&\leq \E \left[\Vert \nabla\ell(w,X) \Vert^2\right]= 4\E\big[ (\beta(w,X)-\beta(w^{\ast},X))^2  \Vert \nabla \beta(w,X) \Vert^2\big],
\end{split}
\end{equation*}
which implies the desired upper bound.}

\textcolor{black}{
We finally derive (\ref{prelem3}). Observe that
\begin{equation*} \begin{split}
Hf(w)&=\E[H \ell(w,X)]\\
&=2\E\left[\nabla \beta(w,X) (\nabla \beta(w,X))^\top+(\beta(w,X)-\beta(w^{\ast},X)) H \beta(w,X)\right].
\end{split}
\end{equation*}
Therefore,
$$\textnormal{Tr}\left((\sigma(w))^\top Hf(w)\sigma(w)\right)=2(I_1+I_2),$$
where
$$
I_1=\E \left[ \textnormal{Tr}\left((\sigma(w))^\top\nabla \beta(w,X) (\nabla \beta(w,X))^\top\sigma(w) \right)\right] 
$$
and
$$
I_2=\E \left[\textnormal{Tr}\left((\sigma(w))^\top(\beta(w,X)-\beta(w^{\ast},X)) H \beta(w,X) \sigma(w)\right)\right].
$$
We next bound $I_1$ and $I_2$ separately. On the one hand,
\begin{equation*} \begin{split}
I_1
&=\E \left[( \nabla \beta(w,X))^\top\sigma(w) (\sigma(w))^\top \nabla \beta(w,X) \right]\\
&\leq \lambda_{\max}\left(\sigma(w) (\sigma(w))^\top\right) \E\left[\Vert  \nabla \beta(w,X)\Vert^2\right].
\end{split}
\end{equation*}
On the other hand,
\begin{equation*} \begin{split}
I_2&=\E \left[(\beta(w,X)-\beta(w^{\ast},X)) \textnormal{Tr}\left((\sigma(w))^\top H \beta(w,X)\sigma(w) \right)\right] \\
&=\E\left[(\beta(w,X)-\beta(w^{\ast},X)) \textnormal{Tr}\left(H \beta(w,X) \sigma(w)(\sigma(w))^\top \right)\right]\\
&\leq D \E\left[(\beta(w,X)-\beta(w^{\ast},X))  \lambda_{\max}\left(H \beta(w,X) \sigma(w) (\sigma(w))^\top \right)\right]\\
&\leq D \lambda_{\max}\left(\sigma(w) (\sigma(w))^\top\right)\E \left[(\beta(w,X)-\beta(w^{\ast},X))  \lambda_{\max}(H \beta(w,X))\right] .
\end{split}
\end{equation*}
Therefore,
\begin{equation} \begin{split} \label{g2}
g(w) &\leq \frac{\lambda_{\max}\left(\sigma(w) (\sigma(w))^\top\right)}{f(w)} \bigg(\E\left[\Vert  \nabla \beta(w,X)\Vert^2\right] \\
&\qquad +D\E \left[(\beta(w,X)-\beta(w^{\ast},X))  
\lambda_{\max}(H \beta(w,X))\right]\bigg).
\end{split}
\end{equation}
We next bound $\lambda_{\max}\left(\sigma(w) (\sigma(w))^\top\right)$. We have
\begin{equation*} \begin{split}
&\lambda_{\max}\left(\sigma(w) (\sigma(w))^\top \right)= \text{sup}_{\xi \in \R^{d}, \Vert \xi \Vert=1}\Vert \sigma(w) (\sigma(w))^\top \xi\Vert  \\
&\qquad =\text{sup}_{\xi \in \R^{d}, \Vert \xi \Vert=1}\Vert\E\left[(\nabla \ell(w, Z)-\nabla f(w))(\nabla \ell(w, Z)-\nabla f(w))^\top \xi \right] \Vert \\
&\qquad \leq \E \left[\Vert \nabla \ell(w, Z)-\nabla f(w)\Vert^2 \right] \\
&\qquad \leq 2(\E \left[\Vert \nabla \ell(w, Z)\Vert^2 \right]+ \Vert \nabla f(w)\Vert^2).
\end{split}
\end{equation*}
Using the definition of $\E \left[\Vert \nabla \ell(w, Z)\Vert^2 \right]$ and applying Jensen's inequality to $\Vert \nabla f(w)\Vert^2$, we conclude that
\begin{equation*} \begin{split}
\frac{\lambda_{\max}\left(\sigma(w) (\sigma(w)\right)^\top)}{f(w)}\leq \frac{16\E\big[ (\beta(w,X)-\beta(w^{\ast},X))^2  \Vert \nabla \beta(w,X)\Vert^2\big]}{\E\big[ (\beta(w,X)-\beta(w^{\ast},X))^2\big]},
\end{split}
\end{equation*}
which together with (\ref{g2}) implies the desired upper bound.}
\end{proof}

\begin{proof}[Proof of Theorem \ref{nn1}]
\textcolor{black}{First observe that the function $\nabla f(w)$ is given by
\begin{equation*} 
 \nabla f(w)= 2\E\left[ (\beta(w,X)-\beta(w^{\ast},X)) \, \nabla \beta(w,X)\right],
\end{equation*}
and the matrix $\sigma(w)$ is given by the unique square root of the covariance matrix of
$$
\nabla\ell(w,X)=2(\beta(w,X)-\beta(w^{\ast},X)) \nabla \beta(w,X).
$$
Therefore, they are locally Lipschitz since they are continuous and differentiable with locally bounded derivatives. Hence, Assumption \ref{a1} holds. Moreover, by Assumption \ref{a3}, $f$ attains its minimum value and $\mathcal{S}$ defined in (\ref{snet}) is the set of minima of $f$, thus Assumption \ref{a1B} holds.}

\textcolor{black}{
Consider an initial condition $w_0=(W^0_1, b^0_1,,\ldots,W^0_L,b^0_L)\in \R^D$ satisfying Assumption \ref{a6}; that is, $W^0_1=0$, $b^0_{\ell}=0$ for all $\ell$, and the entries of $W_2^0,\ldots W_L^0$ are nonnegative. In particular, since $v_{\ell}(0)=0$ for all $\ell$ (Assumption \ref{a7}), we have $\beta(w_0,X)=0$. Since $\beta$ is continuous and $X$ is bounded by $K>0$ a.s. (Assumption \ref{a5}), this implies
\begin{equation} \label{pr1}
f(w_0)=\E \left[ (\beta(w^{\ast},X))^2\right] \leq \max_{x \in \R^d: \Vert x \Vert \leq K} (\beta(w^{\ast},x))^2.
\end{equation}
Let $\gamma>0$ be the minimum of all entries of $W^0_2, \ldots, W^0_{L-1}$ and let $M>0$ be the maximum of all entries of $W^0_2, \ldots, W^0_{L}$. Let $w=(W_1, b_1,\ldots,W_L, b_L)\in \R^D$ such that
$\Vert w -w_0 \Vert \leq \gamma/2$. Then the entries of $W_2, \ldots, W_{L-1}$ are all bounded from below by $\gamma/2$ and the entries of $W_2, \ldots, W_{L}$ are bounded from above by $M'=M+\gamma/2$.
Moreover, the absolute value of each entry of $W_1$ and each entry of each $b_{\ell}$ is bounded from above by $\gamma/2$. Let $x \in \R^d$ with $\Vert x \Vert \leq K$. Then the absolute value of each entry of each $g_1(x,w)$ is bounded from above by $a_1:=\gamma(K+1)$. 
Proceeding inductively as in \cite{C22}, we get that for each $\ell \geq 2$, the 
absolute value of each entry of each $g_{\ell}(x,w)$ is bounded from above by  
$$
a_{\ell}:=\varphi_{\ell-1}(\varphi_{\ell-2}\cdots(\varphi_2(\varphi_1(a_1)M'd_1+\gamma)M' d_2 +\gamma)+\cdots)M'd_{\ell-1}+\gamma,
$$
where $\varphi_{\ell}(y):=\max\{v_{\ell}(y), \vert v_{\ell}(-y) \vert\}$.
Thus, each diagonal entry of each $D_{\ell}(x,w)$ is bounded from below by
$$
c_{\ell}:=\min_{\vert y \vert \leq a_{\ell}} v'_{\ell}(y)>0.
$$
Now let $N>\gamma/2$ be a lower bound on the entries of $W_L^0$.
Then the entries of $W_L$ are bounded from below by $N-\gamma/2$, and hence, for each $i \in \{1,\ldots,d_1\}$,
$$
\min_{x \in \R^d: \Vert x \Vert \leq K} (q_i(w,x))^2 
\geq \left( (N-\gamma/2) (\gamma/2)^{L-2} d_{L-1} \cdots d_2 c_{L-1} \cdots c_1 \right)^2.
$$
The lower bound in equation (\ref{lambda}) gives
\begin{equation}  \label{pr2}
A_{\min}(\gamma/2,w_0)\geq 4  \lambda_{\min}(\Sigma_X) d_1 \left( (N-\gamma/2) (\gamma/2)^{L-2} d_{L-1} \cdots d_2 c_{L-1} \cdots c_1 \right)^2.
\end{equation}}

\textcolor{black}{
On the other hand, since $\beta$ is a $\mathcal{C}^2(\R^D \times \R^d)$ function, the upper bound in equation (\ref{prelem3}) shows that 
$
G_{\max}(\gamma/2,w_0) 
$
is bounded from above by a constant that depends only on $M'$, $K$, $\gamma$, and $v_1,\ldots, v_L$.
Thus, we conclude that there exists $\eta>0$ satisfying assumption (\ref{etay}).}

\textcolor{black}{
Choose $\eta<\frac{A_{\min}(\gamma/2,w_0)}{2G_{\max}(\gamma/2,w_0)}$. 
In particular, such a choice of $\eta$ implies 
\begin{equation} \label{pr3}
\theta(\gamma/2,w_0,\eta)>A_{\min}(\gamma/2,w_0)/2.
\end{equation}
On the other hand, similarly as above, the upper bound in equation (\ref{eq:B}) shows that $B_{\max}(\gamma/2,w_0)$ is bounded from above by a constant that depends only on $M'$, $K$, $\gamma$, and $v_1,\ldots, v_L$. Therefore, we can choose $\eta$ sufficiently small so that
\begin{equation} \label{pr4}
\eta B_{\max}(\gamma/2,w_0)\leq \frac{1}{4},
\end{equation}
which is precisely condition (\ref{bmax}).}

\textcolor{black}{
Using (\ref{pr1}), (\ref{pr2}), (\ref{pr3}), and (\ref{pr4}), we get that $p$ defined in assumption (\ref{heta}) satisfies
\begin{equation*} \begin{split}
p &\leq \frac{2\max_{\Vert  x \Vert  \leq K} \vert  \beta(w^{\ast},x) \vert}{\gamma/2 \sqrt{A_{\min}(\gamma/2,w_0)/2}}(1+1+\frac12) \\
&\leq \frac{8 \max_{\Vert  x \Vert  \leq K} \vert  \beta(w^{\ast},x) \vert}{ \gamma \sqrt{\lambda_{\min}(\Sigma_X) d_1} (N-\gamma/2) (\gamma/2)^{L-2} d_{L-1} \cdots d_2 c_{L-1} \cdots c_1}.
\end{split}
\end{equation*}
Therefore, taking $N$ sufficiently large, condition (\ref{heta}) holds, and since $p$ can be made sufficiently small, applying Theorem \ref{nn1} with $r=\gamma/2$ completes the proof of the first two statements of the theorem.}

\textcolor{black}{
Finally, to show \eqref{ratebis}, note that inequalities (\ref{rate}) and (\ref{pr3}) imply that
for all $\epsilon>0$ and $t>0$,
\begin{equation*} 
\P \left( \Vert w_{t}-x^{\ast} \Vert >\epsilon \vert \tau_{\gamma/2} \wedge \tau =\infty\right) 
\leq \frac{\gamma}{2\epsilon} e^{- \lambda_{\min}(\Sigma_X) d_1 \left( (N-\gamma/2) (\gamma/2)^{L-2} d_{L-1} \cdots d_2 c_{L-1} \cdots c_1 \right)^2 t},
\end{equation*}
which proves \eqref{ratebis} and concludes the proof of the theorem.}
\end{proof}

\acks{\textcolor{black}{We thank two anonymous reviewers for their insightful remarks that led to significant improvements.}
Both authors acknowledge support from the Spanish MINECO grant PID2022-138268NB-100 and
  Ayudas Fundacion BBVA a Proyectos de Investigaci\'on Cient\'ifica 2021.}


\bibliography{sgd}

\end{document}